\documentclass[12pt]{article}
\usepackage{amsmath}
\usepackage{amsfonts}
\usepackage{amsthm}
\usepackage{stmaryrd}
\usepackage{graphicx}
\usepackage{booktabs}
\newtheorem{theorem}{Theorem}
\newtheorem{definition}{Definition}
\newtheorem{proposition}{Proposition}
\newtheorem{lemma}{Lemma}
\newtheorem{observation}{Observation}

\newcommand{\supp}{\text{supp}}
\newcommand{\pdf}{f}
\newcommand{\pdfn}{f_-}
\newcommand{\pdfp}{f_+}
\newcommand{\proj}[2] {#1^T#2}
\newcommand{\cipSymbol} {\text{ip}^\times}
\newcommand{\cip}[1] {\cipSymbol(#1)}
\newcommand{\crossEntropySymbol}{\text{H}_2^\times}
\newcommand{\crossEntropy}[1] {\crossEntropySymbol(#1)}
\newcommand{\entropy}[1] {\text{H}_2(#1)}
\newcommand{\dcsSymbol}{\text{D$_{\text{CS}}$}}
\newcommand{\dcs}[1] {\dcsSymbol(#1)}

\begin{document}

\author{Wojciech Marian Czarnecki}
\date{Faculty of Mathematics and Computer Science\\ Jagiellonian University\\ ul. Lojasiewicza 6, 30-348 Krakow\\ e-mail: {\it w.czarnecki@uj.edu.pl}}

\title{On the consistency of \\Multithreshold Entropy Linear Classifier}

\maketitle

\abstract{Multithreshold Entropy Linear Classifier (MELC) is a recent classifier idea which employs information theoretic concept in order to create a multithreshold maximum margin model. In this paper we analyze its consistency over multithreshold linear models and show that its objective function upper bounds the amount of misclassified points in a similar manner like hinge loss does in support vector machines. For further confirmation we also conduct some numerical experiments on five datasets.}

\section{Introduction}
Many of the existing machine learning classifiers are based on the minimization of some additive loss function which penalizes each missclassification~\cite{scholkopf2002learning}. This class of models consists perceptron, neural networks, logistic regression, linear regression, support vector machines (both traditional and least squares) and many others. For most of such approaches it is possible to prove their consistency, meaning that under assumption that our data is sampled i.i.d. from some unknown probability distributions, algorithm will converge to the optimal model in Bayesian sense with the sample size growing to infinity~\cite{steinwart2002influence, steinwart2005consistency}. While this is quite natural to be consistent with a loss function which is being directly minimized, it generally only upper bounds the number of wrong answers. 

In general, up to some weighting schemes, the classic measure of the classification error is the expected number of missclassified samples from some unknown distribution $\mathcal{F}$:
$$
\mathbb{E}[y_i \neq cl(x_i) | (x_i,y_i) \sim \mathcal{F} ],
$$
which directly translates to
$$
\int  l(cl(x),y,x) p(x) dx,
$$
for $l(p,y,x) = 1 \iff py \geq 0$. We call $l$ the 0/1 loss function and use the $l_{0/1}$ notation. As a result we can define an empirical risk over the training set as
$$
\mathcal{R}_{emp}(\{(x_i,y_i)\}_{i=1}^N) = \tfrac{1}{N} \sum_{i=1}^N l(cl(x_i), y_i, x_i),
$$
which can be minimized over some family of classifiers $cl$. Unfortunately for 0/1 loss the resulting optimization problem is hard even for linear models. To overcome this issue many classifiers are constructed through optimization of some similar loss function which results in feasible problems. For example support vector machines change 0/1 loss to so called hinge loss
$$
l_H(p,y,x) = \max \{ 0, 1-py\},
$$
for $y \in \{-1,+1\}$. It appears, that such problem in the class of linear classifiers is convex and so -- easy to compute. There are two important aspects of hinge loss that make it a reasonable surrogate function. First, $l_H(p,y,x)=0 \rightarrow l_{0/1}(p,y,x)=0$\footnote{Implication is an equivalence relation up to scaling of the linear operator as hinge loss returns non-zero values for predictions in $(-1,1)$ interval.} second $l_H(p,y,x) \geq l_{0/1}(p,y,x)$. In other words, it is an upper bound of the 0/1 loss and when it attains zero then there are no missclassified points.

In this paper we analyze Multithreshold Entropy Linear Classifier, a recently proposed~\cite{czarnecki2014multithreshold} classifier which builds a multithreshold linear model using information theoretic concepts. It is a density based approach which cannot be easily translated to the language of additive loss functions. We show that this model is consistent with 0/1 loss over simple families of distributions and that in general it also upper bounds the 0/1 loss in the class of multithreshold linear classifiers and when it attains zero then there are no missclassified points. We also draw some intuitions to show how this model is related to other linear classifiers and conclude with some numerical experiments.
\section{Multithreshold Entropy Linear Classifier}
Multithreshold Entropy Linear Classifier (MELC~\cite{czarnecki2014multithreshold}) is aimed at finding such linear operator $v$ that maximizes the Cauchy-Schwarz Divergence~\cite{jenssen2006cauchy} of kernel density estimation of each class projection on $v$. It appears that due to the affine transformation invariance of such problem one can (and should, as shown in~\cite{czarnecki2014multithreshold}) restrict to the unit sphere, meaning that $\|v\|=1$.

There are many density based methods in particular one can perform kernel density estimation of any dataset and simply classify according to which density is bigger. However, such an approach cannot work in general due to the \textit{curse of dimensionality} and the fact that density estimation requires enormous number of points for reasonable results (number of required points grows exponentially with the data dimension). As a result, existing datasets can be used to approximate density to at most few dimensions while data can have thousands. This leads to a very natural concept of performing density estimation of low dimensional data projection, in particular one dimensional one, performed by MELC. 

For a given set of points $X_-, X_+$, its projection to the hyperplane $v$ is simply $\proj{v}{X_-}, \proj{v}{X_+}$. Kernel density estimations using Silverman's rule \cite{silverman} is given by
$$
\llbracket \proj{v}{X_\pm} \rrbracket(x) := \tfrac{1}{|X_\pm|} \sum_{x_\pm \in X_\pm} \tfrac{1}{\sqrt{2 \pi} |X_\pm| } \exp \left ( -\tfrac{\| \proj{v}{x_\pm} - x \|^2}{2\sigma_\pm^2} \right ),
$$
where
$$
\sigma_\pm = (\tfrac{4}{3|X_\pm|})^{1/5} \text{std}(\proj{v}{X_\pm}).
$$
Now to define the MELC objective function, we need some definitions, namely:
\begin{itemize}
 \item cross information potential which, as shown in~\cite{czarnecki2014multithreshold}, is connected to minimization of the empirical risk
 $$ \cip{ \pdfn, \pdfp} = \int \pdfn(x)\pdfp(x) dx.$$
 \item Renyi's quadratic cross entropy as defined in~\cite{principe2010information} is simply a negative logarithm of $\cipSymbol$
  $$ \crossEntropy{ \pdfn, \pdfp} = -\ln ( \cip{ \pdfn, \pdfp} ).$$
 \item Renyi's quadratic entropy is a Renyi's quadratic cross entropy between pdf and itself
  $$ \entropy{ \pdf} = \crossEntropy{ \pdf, \pdf}.$$
  \item Cauchy-Schwarz Divergence, optimized by the full MELC model
  $$ \dcs{ \pdfn, \pdfp} =   2 \crossEntropy{\pdfn, \pdfp} - \entropy{ \pdfn} - \entropy{ \pdfp}.$$
\end{itemize}
%
%
In particular, non-regularized MELC is prone to overfitting which can be easily summarized by the following observation.
\begin{observation}
 Given an arbitrary finite, consistent set of samples $\{(x_i,y_i)\}_{i=1}^N \subset \mathbb{R}^d \times \{-1,+1\}$ non-regularized MELC learns it with zero error for sufficiently small $\sigma$.
\end{observation}
\begin{proof}
 First let us notice, that any finite, consistent sample set is separable by some multithreshold linear classifier. In other words 
 $$
 \forall_{\{(x_i,y_i)\}_{i=1}^N} \exists_v \forall_{i,j} \langle v, x_i \rangle \neq \langle v, x_j \rangle
 $$
 Obviously, there are $N^2$ pairs of vectors which can violate this assumption. Each defining a family of linear projections that are projecting them at the same point. 
 $
 \bar v_{ij} = \{ v : \langle v, x_i \rangle = \langle v, x_j \rangle \} = \{ v : \langle v, x_i - x_j \rangle = 0\},
 $
 thus $\forall_{ v_1,v_2 \in  \bar v_{ij}} \exists_{a \in \mathbb{R}} v_1 = av_2.$
 
 So it is sufficient to choose $v \in \mathbb{R}^d \setminus \bigcup_{i,j} v_{ij}$ which is a non-empty set as for any $d>1$ there are infinitely many possible angles that vectors can form with each axis, and for $d=0$ all $v_{ij}=0$ (from the dataset consistency).
 
 In the worst case it results in a $(N-1)-$multithreshold linear classifier. As a consequence, there exists such linear projection for which the smallest margin between samples of this set is greater than zero.

 As it has been shown in~\cite{czarnecki2014multithreshold} non-regularized MELC maximizes the smallest margin among all margins in multithreshold linear classifiers as $\sigma$ approaches $0$. In the same time MELC will not learn these samples perfectly if and only if at least two samples are projected at the very same point, which is equivalent to the maximum of the smallest margin in the class of multithreshold linear classifiers for this sample is equal to $0$, contradiction.
\end{proof}
In particular, this means that for small values of $\sigma$, without regularization, this model has infinite Vapnik-Chervonenkis dimension~\cite{vapnik2000nature}, as many other density or nearest neighbours based approaches. In the following section we focus on more practical characteristics - whether this classifier is able to learn an arbitrary continuous distribution with smallest obtainable error in its class of models. This characteristic is called \textit{consistency} and can be defined as
\begin{definition}[Consistency]
 Model $M$ is called consistent with error measure $E$ and family of distributions $\mathcal{F}$ in the class of models $\mathcal{M}$ if for any $f \in \mathcal{F}$ $M$ trained on the i.i.d. samples from $f$ approaches minimum error as measured by $E$ over all models in  $\mathcal{M}$ on $f$ with samples' size going to infinity.
\end{definition}
\section{Non-regularized MELC consistency}
In this section we focus on non-regularized MELC which searches for linear projection $v$ (with norm 1) maximizing Renyi's quadratic cross entropy of kernel density estimation of data projection:
$$
v_{\crossEntropySymbol} = \arg\max_v \crossEntropy{\llbracket\proj{v}{X_-}\rrbracket, \llbracket\proj{v}{X_+}\rrbracket},
$$
which makes a classification decision based on the estimated projected densities
$$
cl(x) = \text{sign}(\llbracket\proj{v_{\crossEntropySymbol}}{X_+}\rrbracket(x) - \llbracket\proj{v_{\crossEntropySymbol}}{X_-}\rrbracket(x) ).
$$
We show that such classifier is nearly consistent with the 0/1 loss in the class of all multithreshold linear classifiers. We also draw an analogy between its approach to the one taken by support vector machines model (as well as other regularized empirical risk loss function minimization based models). Let us start with some basic definitions and notations.
\begin{definition}[Expected accuracy]
 Given some classifier $cl(x) : X \rightarrow \{-1,+1\}$ the expected accuracy over a distributions $\pdfn, \pdfp$ with priors $p(-), p(+)$ is 
 $$
  p(-) \int \max\{0, -cl(x)\}\pdfn(x) dx  + p(+) \int \max \{ 0, cl(x) \}\pdfp(x) dx.
 $$
\end{definition}
\noindent For unbalanced datasets we might be more interested in measures that make both classes equally important despite their sizes (priors) which leads to the averaged accuracy (also known as balanced/weighted accuracy).
\begin{definition}[Expected averaged accuracy]
 Given some classifier $cl(x) : X \rightarrow \{-1,+1\}$ the expected averaged accuracy (ignoring the classes' priors) over a distributions $\pdfn, \pdfp$ is 
 $$
 \tfrac{1}{2} \int \max\{0, -cl(x)\}\pdfn(x) dx  + \tfrac{1}{2} \int \max \{ 0, cl(x) \}\pdfp(x) dx.
 $$
\end{definition}
\noindent  Let us now compute the smallest obtainable error by multithreshold linear classifiers as measured by expected averaged accuracy (EAA).
\begin{proposition}[Multithreshold Linear Classifier EAA Bayes Risk]
 For the family of multithreshold linear classifiers, the smallest obtainable EAA error for distributions $\pdfn, \pdfp$ equals
 $$
 \mathcal{R}_\text{EAA}(\pdfn, \pdfp) = \min_v \int \min\{(\proj{v}{\pdfn})(x), (\proj{v}{\pdfp})(x)\} dx.
 $$
\end{proposition}
\begin{proof}
$\int \min\{(\proj{v}{\pdfn})(x), (\proj{v}{\pdfp})(x)\} dx$ simply expresses the probability of making a bad classification over whole data projection. For each point $v^Tx$, we have to classify it as a member of either $\pdfn$ or $\pdfp$ and obviously, we make an error when classifying any point $x$ with probability $\min\{(\proj{v}{\pdfn})(x), (\proj{v}{\pdfp})(x)\}$. As a result, the projection which realizes the minimum of probability of an error is the one giving the greatest expected averaged accuracy.
\end{proof}
In the following sections we assume that the kernel density estimation approximating the data distribution is the actual distribution, as with the sample size growing to infinity kernel density estimation with Silverman's rule~\cite{silverman} is guaranteed to converge to the true distribution. As a consequence each result regarding a property over distribution is also true over finite sample in the limiting case. We also use the notation 
$$
\mathcal{R}_\text{EAA}(v; \pdfn, \pdfp) = \int \min \{ \proj{v}{\pdfn}(x), \proj{v}{\pdfp}(x)\} dx,
$$
for the smallest obtainable multithreshold linear classifier missclassification error for a given projection $v$. So in particular
$$
v_{opt}  =\arg\min_v \mathcal{R}_\text{EAA}(v; \pdfn ,\pdfp)
$$
$$
\mathcal{R}_\text{EAA}(\pdfn ,\pdfp )= \min_v \mathcal{R}_\text{EAA}(v; \pdfn ,\pdfp) = \mathcal{R}_\text{EAA}(v_{opt}).
$$
Let us begin with the simplest case, when there exists a perfect classifier able to distinguish samples' classes (case when Bayesian risk is 0). 
\begin{observation}
Non regularized MELC is consistent with 0/1 loss on multithreshold linearly separable distributions.
\end{observation}
\begin{proof}
If two distributions are perfectly separable by a multithreshold linear separator then there exists a linear projection $v_{opt}$ such that common support of distributions projected on $v_{opt} $ has zero measure.
$$
| \supp(v_{opt}^T f_-) \cap \supp(v_{opt}^T f_+) | = 0.
$$
Obviously, $\cip{v_{opt}^T f_-, v_{opt}^T f_+} = 0$ as we integrate the function which is not equal to 0 only on the set o zero measure. 

Similarly $\forall v : \cip{v^T f_-, v^T f_+} = 0 \rightarrow | \supp(v^T f_-) \cap \supp(v^T f_+) | = 0 $ because if the integral of the product of two functions is equal to zero then only on the set of zero measure both of these functions can be non-zero. As a result the solution given by non-regularized MELC attains the Bayesian risk for this class of distributions.
\end{proof}
Let us now investigate the situation when data of each class come from a radial normal distributions. 
\begin{observation}
Non regularized MELC is consistent with 0/1 loss on radial normal distributions.
\end{observation}
\begin{proof}
Let us assume that we are given Gaussians with variances $\sigma_-^2$ and $\sigma_+^2$ respectively.
$$\pdfn = \mathcal{N}(m_-, \sigma_-^2 I), \pdfp = \mathcal{N}(m_+, \sigma_+^2 I)$$
It is easy to see that linear projections of these distributions form the family of one-dimensional normal distributions with variances $\sigma_-^2, \sigma_+^2$ respectively and distance between their means in the $[0, \| \proj{v}{m_-} - \proj{v}{m_+}\|]$ interval. Optimal projection is given by $v_{opt}$ which maximizes the distance between these means, so $v_{opt}= \pm( m_- - m_+ ) $.

On the other hand according to Czarnecki et al.~\cite{czarnecki2014multithreshold}, we have
$$
\cip{v^Tf_-, v^Tf_+} = \tfrac{1}{\sqrt{2\pi(\sigma_-^2+\sigma_+^2 )}} \exp\left (-\frac{\|v^Tm_--v^Tm_+\|^2}{2(\sigma_-^2+\sigma_+^2)} \right ),
$$
so obviously $\cipSymbol$ is minimized (and $\crossEntropySymbol$ maximized) when $\| v^Tm_- - v^Tm_+ \|^2$ is maximized. As a result non-regularized MELC selects optimal linear projection. 
\end{proof}


Unfortunately MELC (neither regularized nor non-regularized) does not seem to be consistent with 0/1 loss in general. However, we show that 0/1 loss is nicely bounded by its objective function which will draw an analogy between this approach and those taken by other linear models.

We start with a simple lemma connecting square of the function's integral and integral of the function's square on a bounded interval.
\begin{lemma}
\label{lem:cauchy}
For any square integrable function $f$ such that $\forall x : f(x) \geq 0$
 $$
 \int_{0}^{1} f(x) dx  \leq \sqrt{ \int_{0}^{1} f^2(x)  dx}.
 $$
\end{lemma}
\begin{proof}
 This is an obvious consequence of Schwarz inequality
 $$
 \left ( \int_a^b f(x)g(x) dx  \right )^2 \leq { \int_a^b f^2(x) dx \int_a^b g^2(x) dx },
 $$
for $a=0, b=1$, $f$ being non-negative and $g$ being a constant function equal $c>0$,
 $$
 \int_{0}^{1} f(x) dx = \frac{1}{c} \int_{0}^{1} c \cdot f(x) dx  \leq \frac{1}{c} \sqrt{ \int_{0}^{1} c^2 dx\int_{0}^{1} f^2(x) dx  } = \sqrt{ \int_{0}^{1} f^2(x) dx }.
 $$
\end{proof}
Now we can prove the main theorem of this paper.
\begin{theorem}


Negative log likelihood of minimal obtainable missclassification error of a given multithreshold linear classifier for any not multithreshold linearly separable distributions is at least half of Renyi's quadratic cross entropy of data projections used by this classifier.
\end{theorem}
\begin{proof}
 First from the fact that we can scale/center data so for any linear operator $v$ such that $\|v\|=1$ we have
$$
0 \leq \sup( \supp(\proj{v}{\pdfn})  \cup  \supp(\proj{v}{\pdfp}) ) - \inf( \supp(\proj{v}{\pdfn})  \cup  \supp(\proj{v}{\pdfp}) ) \leq 1,
$$
 and consequently we can narrow down to the error over a unit interval\footnote{for KDE based on functions with infinite support, for a proper scaling, integral of the pdf outside $[0,1]$ interval goes to $0$ with samples size growing to infinity}
 . From Lemma~\ref{lem:cauchy} we get
 \begin{equation}
 \int_{0}^{1} \min\{(\proj{v}{\pdfn})(x), (\proj{v}{\pdfp})(x)\} dx  \leq \sqrt{ \int_{0}^{1} (\min\{(\proj{v}{\pdfn})(x), (\proj{v}{\pdfp})(x)\})^2 dx }.
 \label{th:1part1}
 \end{equation}
 For any ${a,b \in \mathbb{R}_+}$ we have $\min \{a,b\} \leq \sqrt{ab}$, thus
 $$
 \min\{(\proj{v}{\pdfn})(x), (\proj{v}{\pdfp})(x)\}  \leq  \sqrt{(\proj{v}{\pdfn})(x)(\proj{v}{\pdfp})(x)},
 $$
  which connected with (\ref{th:1part1}) yields
 $$
 \mathcal{R}_\text{EAA}(v; \pdfn, \pdfp) = \int_{0}^{1} \min\{(\proj{v}{\pdfn})(x), (\proj{v}{\pdfp})(x)\} dx \leq \sqrt{ \int_{0}^{1} (\proj{v}{\pdfn})(x)(\proj{v}{\pdfp})(x) dx },
 $$
consequently, as $\pdfn, \pdfp$ are not multithreshold linearly separable, $\mathcal{R}_\text{EAA}(v; \pdfn, \pdfp)$ is strictly positive, thus
$$
- \ln( \mathcal{R}_\text{EAA}(v; \pdfn, \pdfp) ) \geq - \ln \left ( \sqrt{ \int_{0}^{1} (\proj{v}{\pdfn})(x)(\proj{v}{\pdfp})(x) dx } \right ) = \tfrac{1}{2} \crossEntropy{\proj{v}{\pdfn}, \proj{v}{\pdfp}}.
$$
\end{proof}
In other words by maximizing the Renyi's quadratic cross entropy (minimizing the cross information potential) we should also optimize negative log likelihood of correct classification (get close to the Bayes risk of 0/1 error). It is worth noting that we do not assume any particular kernel so even though MELC is defined with Gaussian mixtures kernel density estimation, the theorems holds for any square integrable distributions on $[0,1]$ interval.

%



\begin{figure}[!Hb]
\centering
 
 \includegraphics[width=2.7cm]{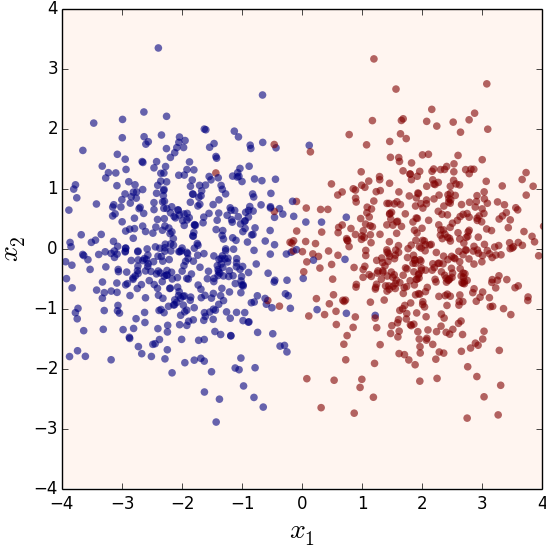}
 \includegraphics[width=3.2cm]{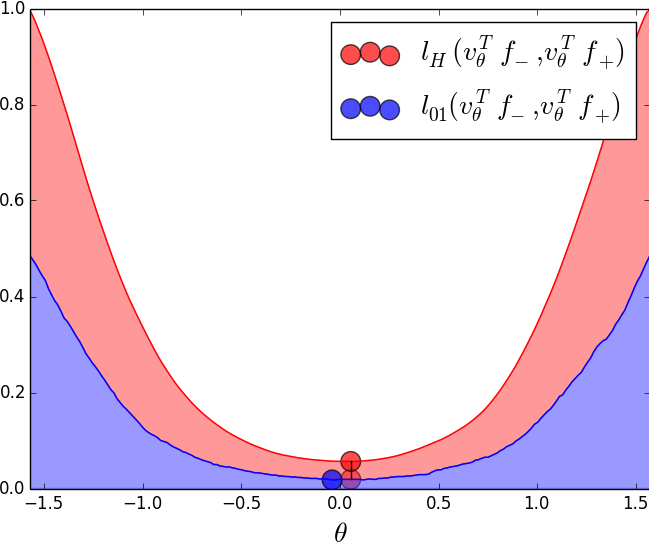}
 \includegraphics[width=2.7cm]{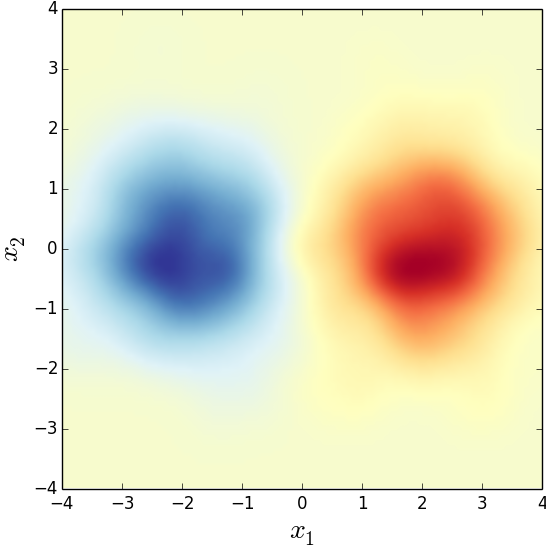}
 \includegraphics[width=3.2cm]{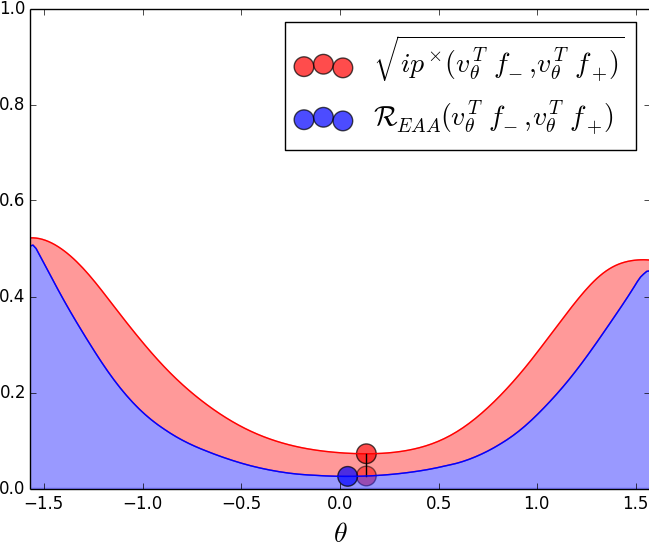} \\
 
 \includegraphics[width=2.7cm]{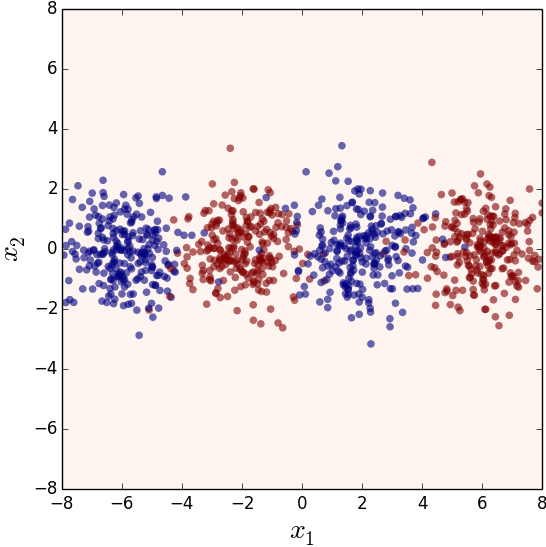}
 \includegraphics[width=3.2cm]{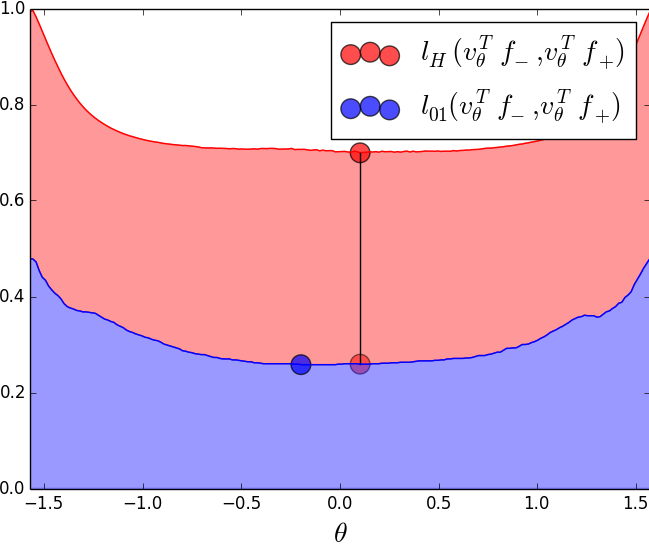}
 \includegraphics[width=2.7cm]{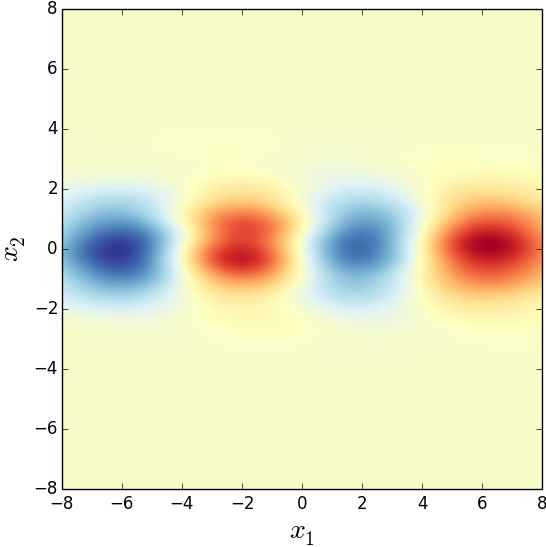}
 \includegraphics[width=3.2cm]{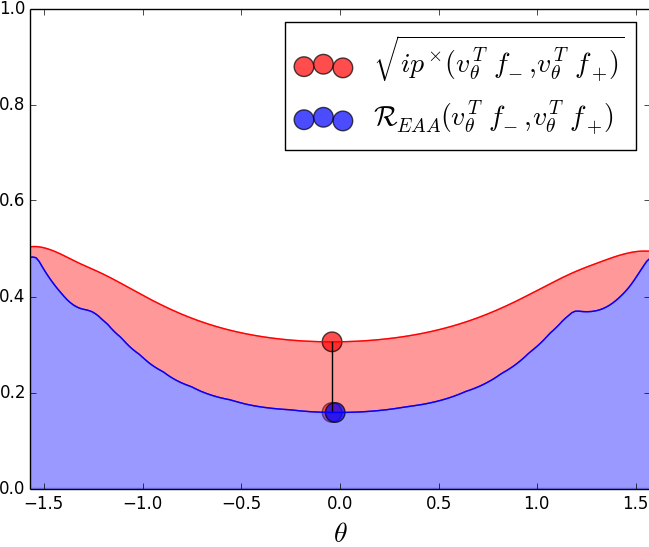} \\
 
 \includegraphics[width=2.7cm]{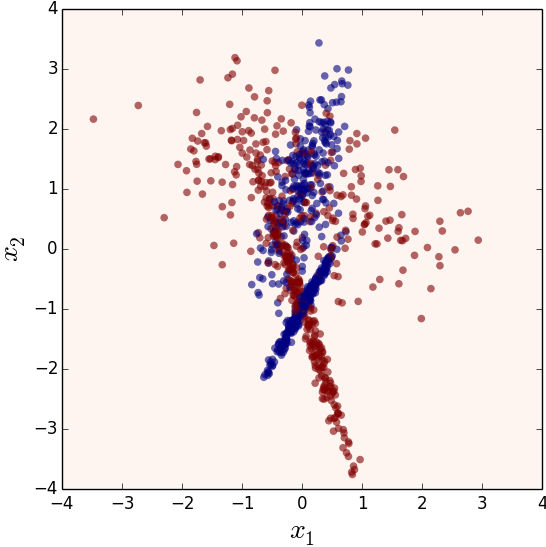}
 \includegraphics[width=3.2cm]{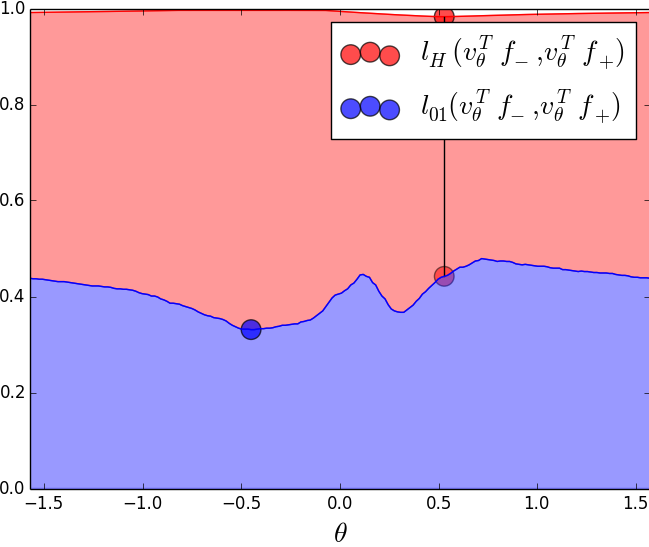}
 \includegraphics[width=2.7cm]{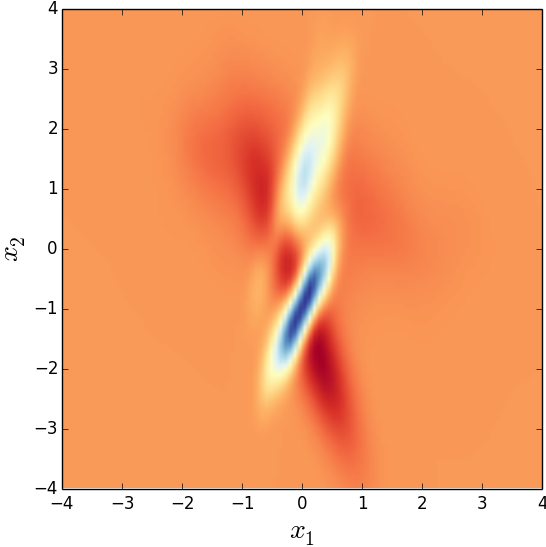}
 \includegraphics[width=3.2cm]{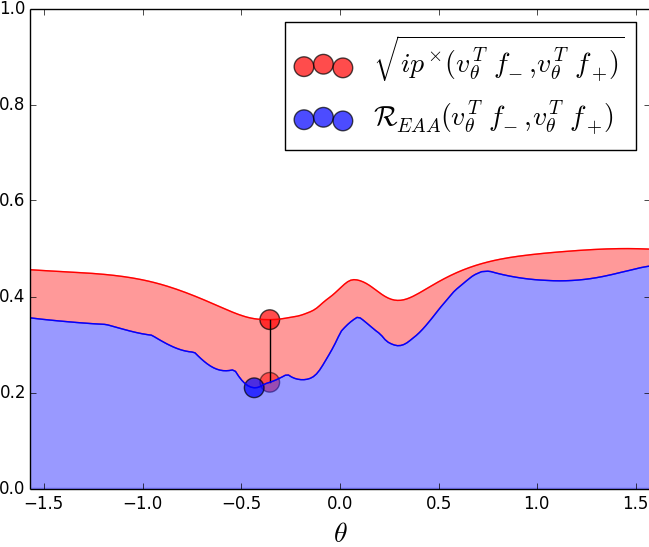} \\

 \includegraphics[width=2.7cm]{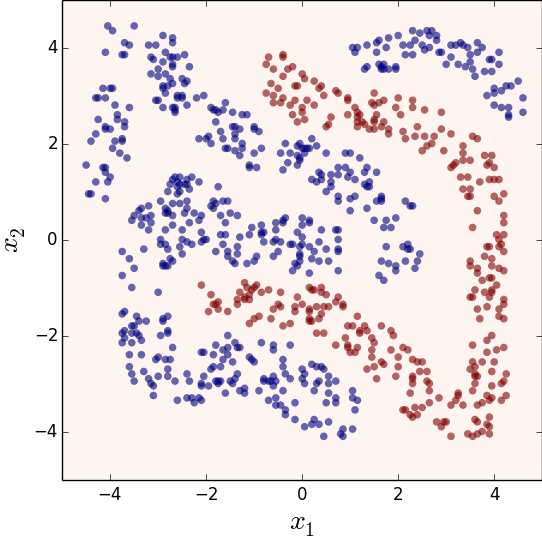}
 \includegraphics[width=3.2cm]{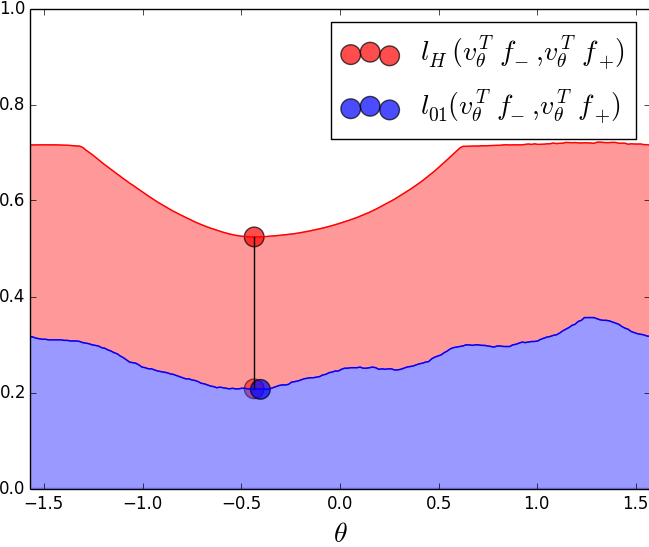}
 \includegraphics[width=2.7cm]{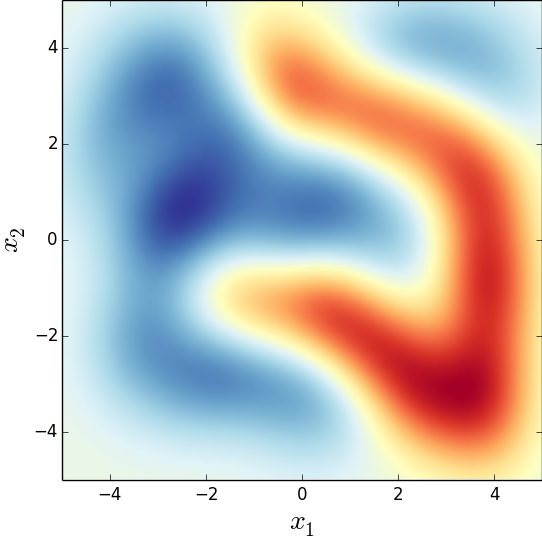}
 \includegraphics[width=3.2cm]{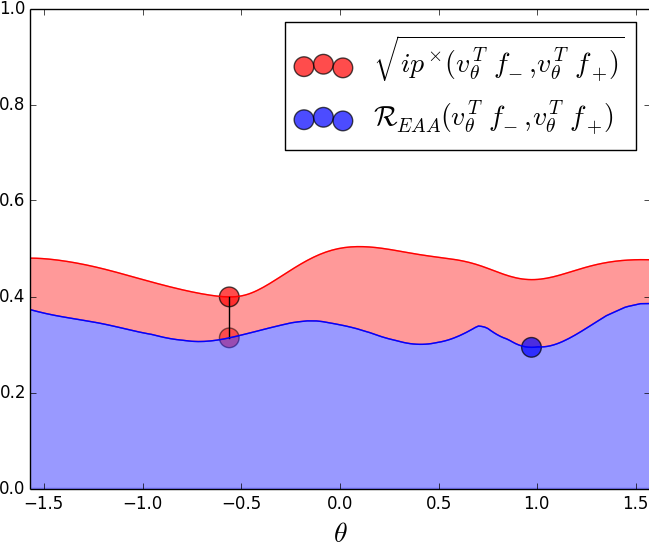} \\

 \includegraphics[width=2.7cm]{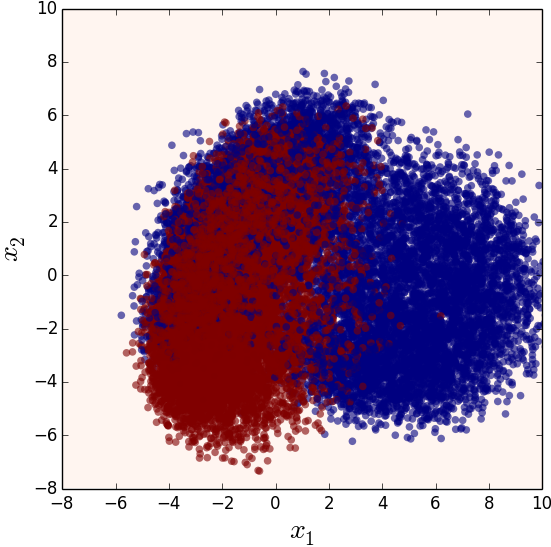}
 \includegraphics[width=3.2cm]{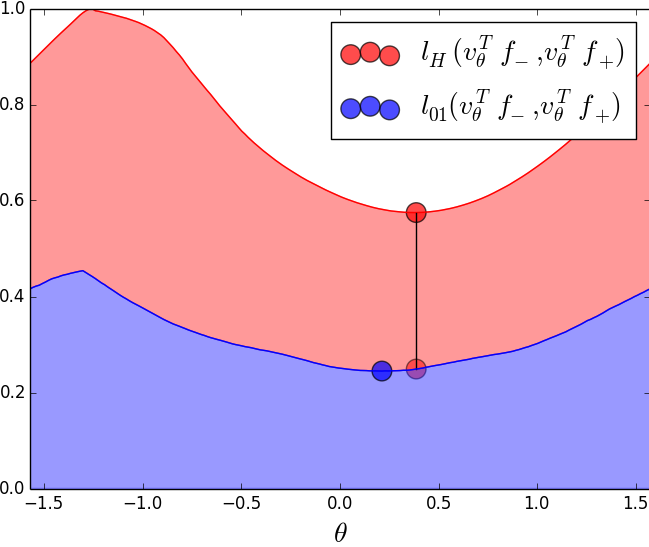}
 \includegraphics[width=2.7cm]{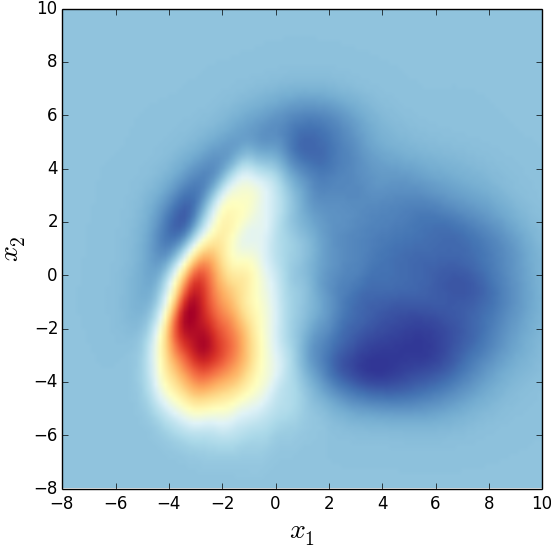}
 \includegraphics[width=3.2cm]{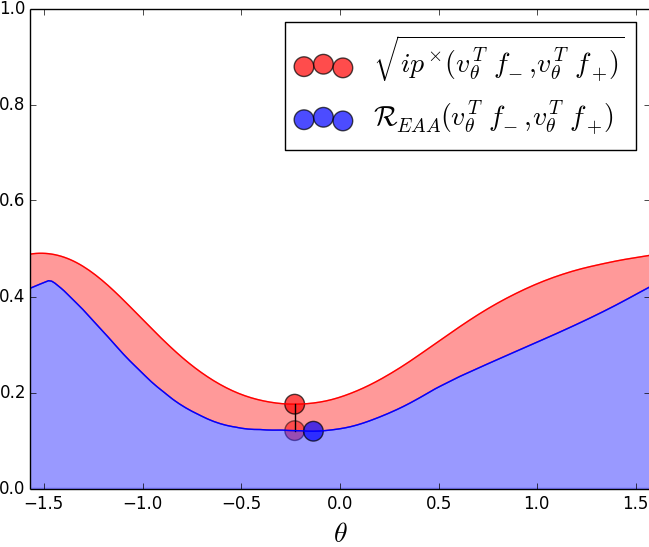} \\

 \caption{Visualization of sampled points for each dataset (first column), hinge loss and Bayesian risk of linear models (second column), underlying dataset distribution (third column) and finally square root of the cross information potential and the Bayesian risk of multithreshold models (last column). X axis corresponds to the angle of the $v$ vector. Large dots correspond to minima of each function, additionally for both hinge loss and $\sqrt{\cipSymbol}$ there is another dot denoting the value of true error obtained if solution is selected using these objectives.}
 \label{fig:obj}
\end{figure}

\section{Experiments}

To further confirm our claims we perform simple numerical experiments on five datasets, three of which are synthetic ones and two real life examples. During this evaluation we analyze all possible linear models in two-dimensional space and compare how particular upper bound objective (hinge loss in the case of linear classifiers and non-regularized MELC for multithreshold classifiers) behaves as compared to the Bayesian risk. Figure~\ref{fig:obj} visualizes the results for: two radial Gaussians distributions (one per class) in 2d space;  four radial Gaussians distributions placed alternately (two per class) in a line; four random strongly overlapping Gaussian distributions (two per class); fourclass dataset~\cite{ho1996building}; 2d PCA embedding of the images of 0 and 2s (positive class) and 3s and 8s from MNIST dataset~\cite{lecun1998mnist}.

First, it is easy to notice the convexity of the hinge loss objective function. Even for problems having multiple local optima (like fourth dataset) the SVM objective function has just one, global optimum which is the core advantage of such an approach. In the same time, non-regularized MELC function has similar number of local optima like the Bayesian risk function, however it is much smoother and as a result one of the unimportant local solution in terms of 0/1 loss in the fourth example (located near $0.5$) is not a solution of MELC.

On the other hand for datasets where the considered class of models is not sufficient (like third problem for linear model) hinge loss convex upper bounds leads to the selection of the point distant from the true optimum (see Table~\ref{tab:comp}). MELC on the other hand seems to better approximate the underlying Bayesian risk function and results in the solutions with comparable error (even if the solution itself is far away from the true optimum, like in the case of fourth dataset).

\begin{table}[htb]
\begin{tabular}{lrrrr}
\toprule
dataset  & $\mathrm{E}(v_{H}, l_{0/1})$ & $\cos(v_{H}, v_{0/1})$  & $\mathrm{E}(v_{\cipSymbol}, \mathcal{R}_\text{EAA})$& $\cos(v_{\cipSymbol}, v_{\mathcal{R}_\text{EAA}})$\\  
\midrule
2 Gauss 2d &
6\% &
  1.00

&
3\% &
  1.00

\\

4 Gauss in line &
0\% &
  0.96

&
0\% &
  1.00

\\

4 Gauss mixed &
34\% &
  0.56

&
5\% &
  1.00

\\

fourclass &
1\% &
  1.00

&
7\% &
  0.05

\\

MNIST &
2\% &
  0.99

&
1\% &
  1.00

\\

\bottomrule
\end{tabular}
\caption{Comparison of solutions given by optimization of hinge loss and optimal linear classifier and between non-regularized MELC and optimal multithreshold linear classifier. Error function is the relative increase in the corresponding error measure when using a particular optimization scheme $\mathrm{E}(m, f) = \tfrac{f(m)-\min_v f(v)}{\min_v f(v)} $. $v_H$ is a linear projection given by hinge loss optimization, $v_{0/1}$ by 0/1 loss optimization, $v_{\cipSymbol}$ by non-regularized MELC and $v_{\mathcal{R}_\text{EAA}}$ the optimal multithreshold linear projection in the Bayesian sense.}
\label{tab:comp}
\end{table}

\section{Conclusions}

In this paper Multithreshold Entropy Linear Classifier is analyzed in terms of its consistency with 0/1 loss function in the class of multithreshold linear classifiers. It has been shown that it is truly consistent with some simple distribution classes and that in general its objective function upper bounds the 0/1 loss in a similar manner as hinge or square losses upper bounds 0/1 loss. Experiments on the synthetic, low dimensional data showed that in practise, one can expect that optimization of MELC objective function truly leads to the nearly optimal classifier with sample size growing to infinity.  
%
%

\bibliographystyle{plain}
\bibliography{biblio}

\end{document}